\documentclass{article}

\usepackage{algorithm}
\usepackage{algpseudocode}

\usepackage{amsmath}
\usepackage{amsthm}
\usepackage{amsfonts}
\usepackage{amssymb}

\usepackage{comment}

\usepackage{tikz}
\usetikzlibrary{arrows}

\newtheorem{thm}{Theorem}[section]

\newtheorem{lem}[thm]{Lemma}

\newcommand{\spaceo}{\hspace{2 mm}}
\newcommand{\setsep}{ \spaceo | \spaceo}

\newcommand{\argmax}{\operatornamewithlimits{argmax}}
\newcommand{\innerp}[2]{\langle #1,#2 \rangle}
\newcommand{\BRK}[1]{\left(#1\right)}
\newcommand{\SETBRK}[1]{\left\{#1\right\}}

\author{
  Mark Kozdoba \\
  \texttt{markk@tx.technion.ac.il}
 \and
 Shie Mannor\\
  \texttt{shie@ee.technion.ac.il}
}

\title{Overlapping Community Detection by Online Cluster Aggregation}
\date{}

\begin{document} 

\maketitle

\begin{abstract}
We present a new online algorithm for detecting overlapping communities. The main ingredients are a
modification of an online k-means algorithm and a new approach to modelling overlap in communities.
An evaluation on large benchmark graphs shows that the quality of discovered communities compares favourably 
to several methods in the recent literature, while the running time is significantly improved. 
\end{abstract}

\section{Introduction}
A community in a graph is a set of nodes such that the density of connections between the nodes within the set 
is higher then then density of connections between the set and its complement. Communities have been observed 
in a wide variety of real world graphs, such as scientific paper citation networks, friendship networks in 
social media, link graphs of the internet, transportation networks and protein-protein interaction networks, 
to name a few. Generally, members of the same community share similar application specific properties and 
communities can be regarded as higher level building blocks of the graphs. 

In many situations it is natural to assume that a node in a graph can belong to several communities. For 
instance, a member on a social network can belong to a community 'Family', to community 'School' and 
to community 'Karate club'. A node in a transportation network can belong to several communities if it is 
a hub on a boundary of two or more regions. 

Community detection is an active research field, and it has created a large and growing literature. We refer to \cite{Fortunato201075} for an extensive survey and a sample of applications, and to \cite{XieSurvey}, which surveys specifically overlapping community detection methods.

As with many data mining problems, one can say that there are two main challenges in community detection. The 
first is to detect communities as precisely as possible. One common approach to measuring this is to run the 
algorithms on a set of LFR benchmarks,\cite{LFRBench},\cite{LFBench}. The LFR benchmarks are models of random graphs 
with a community structure and have certain characteristics resembling real world graphs, such as power law 
degree distributions. More details are given in Section \ref{sec:lfr}. The quality of the communities produced 
by the algorithm is then asserted  by comparing them to the known ground truth communities of the benchmarks,
using the extended mutual normalized information (ENMI) measure, defined in \cite{LFBench}.

The other challenge is to design algorithms that can run on really large graphs in a reasonable time. In 
recent years several methods for detection of overlapping communities were developed that scale to graphs 
with millions of nodes. In particular, as reported in \cite{GopBlei}, the algorithms SVI, due to 
\cite{GopBlei}, the Poisson modelling algorithm due to \cite{NewmanPoi}, the COPRA algorithm due to 
\cite{Greg}, and the INFOMAP algorithm due to \cite{InfomapOvr} can produce non-trivial results on LFR graphs 
with $N=1,000,000$ nodes and about $750$ overlapping communities, each of size $2000$ to $5000$. It was found 
that the SVI and Poisson algorithms produced an ENMI score of $.8$ on these graphs, while COPRA and INFOMAP 
produced ENMI of $.5$ and $.25$ respectively. On the other hand, the running time allocated to the SVI and 
Poisson algorithms was 24 hours, after which the algorithms were terminated and the current best estimate on 
the communities was returned. The running time of COPRA and INFOMAP were not specified. 

In this paper we present a simple online algorithm,CLAGO (Cluster Aggregation for Overlapping Communities), for detecting overlapping communities. We evaluate our algorithm 
on a set of large benchmarks with the same parameters as were considered in \cite{GopBlei} and find that the performance 
in terms of ENMI is similar to the performance of SVI and Poisson algorithms, while the running time of our algorithm 
is significantly better. In particular, our algorithms produces ENMI of $.8$ on the above mentioned $N=1,000,000$ 
benchmark after $2.5$ hours. 

Our algorithm operates in two stages. In the first stage, we produce a non-overlapping partition of a graph. This part of 
the algorithm takes as an input the number of nodes $N$, and the number of communities to find $k$ (however, see later 
remarks about pre-specifying the number of communities). It maintains $k$ vectors of length $N$ as parameters, and it is 
assumed that the algorithm is 
presented with nodes of the graph, one at a time, in a random order. Each node is presented to the algorithm 
together with the set of its neighbours, and for each node an update to the parameters is made. After all nodes were 
presented, we either terminate or proceed to another iteration. The description of the update is given in Section \ref{sec:noc}. Curiously enough, we find that the above 
mentioned $N=1,000,000$ benchmarks contain enough redundancy so that even a single iteration (<30 min.) is 
sufficient to obtain an ENMI as high as $.75$. In the second stage of the algorithm we derive overlapping 
communities from the disjoint communities. The general philosophy is that the probability that a node belongs 
to a given community is proportional to the probability of hitting the node by a random walk started at the 
community. Details are given in Section \ref{sec:oc}.

The rest of the paper is organized as follows: In Section \ref{sec:algorithm} we define the algorithm and derive some of its basic properties. Section \ref{sec:eval} contains the empirical evaluation of the algorithm.

\section{Algorithm}
\label{sec:algorithm}

Suppose we are given a graph $G$ with $N$ vertices, and we want to find $k$ overlapping communities in it. 
The algorithm proceeds in two stages. First we partition the graph into $k$ disjoint communities. 
Then we use a step of a random walk from the disjoint parts to deduce the overlapping communities.

\subsection{Disjoint Communities}
\label{sec:noc}
Denote the graph by $G= (V,E)$. For a node $x\in V$, denote by $d_x$ the degree of $x$, and by 
$n_x$ the set of neighbours of $x$. Set $w_x$ to be a distribution of one step of a random walk from $x$ - 
a uniform measure on $n_x$.  For a subset $S\subset V$, let $d_S = \sum_{x \in S}  d_x$ be the total degree 
of the set $S$. 

Let $p_1,\ldots,p_k$ be randomly initialized probability measures
on $V$. Specifically, we use the uniform disjoint initialization - partition $V$ into $k$ random subsets of 
equal size, $S_1,\ldots,S_k \subset V$, and set $p_i(x) = 1_{S_{i}}(x) \big/ |S_i|$. 

\begin{algorithm}
   \caption{Disjoint Communities, CLAG}
   \label{alg:noc}
\begin{algorithmic}[1]	
	\State Initialize $p_1,\ldots,p_k$
    \State Initialize counters $m_1 = m_2 = \ldots = m_k = 0$
	\Repeat \label{alg:bigloop}
		\State Set $x_1,\ldots,x_N$ to be a random permutation of the nodes.
		\For{ $i \in 1,\ldots,N$ }   \label{alg:for_loop}
			\State  $t \gets \argmax_{ 1\leq j \leq k} \innerp{p_j}{w_{x_i}}$ \label{alg:assignment_line}
			\State  $m_t \gets m_t + d_{x_i}$                    \label{alg:counters_line}
			\State  $p_t \gets \BRK{1 - \frac{d_{x_i}}{m_t}} p_t + \frac{d_{x_i}}{m_t} w_{x_i}$
			                                                                 \label{alg:scaling_line}
		\EndFor
	\Until{Stopping condition is met}           \label{alg:stopping_criterion}
	\State For $t \leq k$, set \\ 
      \spaceo \spaceo  $C_t = \SETBRK{ x \in V \spaceo | \spaceo t = 
	                            \argmax_{ 1\leq j \leq k} \innerp{p_j}{w_{x_i}} }$
	\State Return the partition $C_1,\ldots, C_k$, and the parameters $\{p_j\}, \{m_j\}$. 
\end{algorithmic}
\end{algorithm}

The stopping criterion on line \ref{alg:stopping_criterion} can be chosen in any common way. We have found 
that simply limiting then number of iterations to somewhere from $5$ to $15$ is usually sufficient.

The inner product on line \ref{alg:assignment_line} is crucial to the algorithm. Indeed, if the inner product 
\begin{equation}
	\label{alg:innerp}
	\innerp{p_j}{w_{x_i}}
\end{equation}
is replaced by the squared Euclidean norm, 
\begin{equation}
	\label{alg:eucl_norm}
	-|p_j - w_{x_i}|^2 = -|w_{x_i}|^2 + 2\innerp{p_j}{w_{x_i}} - |p_j|^2
\end{equation}
then Algorithm \ref{alg:noc} is
the online k-means algorithm in a form given in \cite{BotBen} (with a caveat that every $w_x$ comes with 
multiplicity $d_x$).  We find empirically that the usage of (\ref{alg:innerp}) instead of 
(\ref{alg:eucl_norm}) results in significantly better quality of found partitions. Let us mention one possible 
reason why (\ref{alg:innerp}) performs better than (\ref{alg:eucl_norm}). The difference between the 
expressions is the term $|p_j|^2$, which is small for spread-out, large support measures. This can, for 
instance, be observed on the Political Blogs example \cite{PolBlogs} (see Section
 \ref{sec:std_examples}), where Algorithm \ref{alg:noc} with cost (\ref{alg:eucl_norm}) finds one huge 
 component instead of two smaller ones. However, the cost (\ref{alg:innerp}) has bias 
 towards small support measures. As mentioned in Section \ref{sec:olfr}, on large graphs with high $k$, 
 Algorithm \ref{alg:noc},CLAG, tends to produce, in addition to the true partition of the graph, some small 
 components. These components can easily be detected by their size and pruned. 

In the rest of this section we describe several basic properties of CLAG Algorithm. First we discuss 
the shape of the parameters $p_i$ when the algorithm is close to convergence. 

Denote by $p_j^s$ and $m_j^s$ the value of parameters $p_j$ and $m^j$ at the start of iteration $s$ of the 
``repeat" loop
of the algorithm.  For each $x\in V$, set $C^s(x)$ to be the parameter index to which $x$ was assigned during 
iteration $s$. 
In other words, $C^s(x)$ is the value of $t$ that was assigned in line $\ref{alg:assignment_line}$ when $x_i$ 
was $x$, on the $s$'th iteration of the ``repeat" loop. We say that the algorithm is in \textit{stationary 
state} at iteration $s$ if for all $x\in V$ and all $s'>s$, 
\begin{equation}
C^{s'}(x) = C^{s}(x).
\end{equation}
Clearly if the parameters $p_j^s$ converge to some limiting values $\hat{p}_j$, then the algorithm will enter 
the stationary state for large enough $s$ (assuming some consistent rule for breaking ties at line 
\ref{alg:assignment_line}). 

Denote by $\pi$ the stationary measure of a random walk on $G$, 
\begin{equation}
\label{eq:pi_def}
\pi(x) = d_x \big/ d_V. 
\end{equation}
For a subset $C \subset V$, denote by $\pi_C$ the restriction of $\pi$ to $C$. Namely, $\pi_C(x) = d_x \big/ 
d_C$ if 
$x \in C$ and $\pi_C(x) = 0$ otherwise. Set 
\begin{equation}
\label{eq:mu_c_def}
\mu_C = \frac{1}{d_C} \sum_{x\in C} d_x \cdot w_x. 
\end{equation}
Then $\mu_C$ is the distribution of a random walk that was started from $\pi_C$ and performed one step. 

\begin{lem}
\label{lem:limit_shape}
Assume that Algorithm \ref{alg:noc} is at stationary state at some iteration $s$. 
For $j \leq k$, let 
\begin{equation}
	P_j^s = \SETBRK{x\in V \setsep  C_s(x) = j}
\end{equation}
be the set of nodes assigned to $p_j^{s}$. Then for all $j\leq k$,
\begin{equation}
	\label{alg:pj_form}
    p_j^{s'} \rightarrow \mu_{P_j^s} , 
\end{equation} 
and 
\begin{equation}
	\label{alg:nj_form}
    \frac{m_j^{s'}}{\sum_{l \leq k} m^{s'}_l} \rightarrow  d_{P_j^s} \big/ d_V
\end{equation} 
as $s' \rightarrow \infty$. 
\end{lem}
\begin{proof}
By the update rules at lines \ref{alg:counters_line} and \ref{alg:scaling_line}, the following more general 
relation holds for all $s\geq 1$: 
\begin{equation}
\label{eq:pj_recursion}
p_j^{s+1} = \frac{m^s_j}{m^s_j + d_{P_j^s}} p_j^s + \frac{d_{P_j^s}}{m^s_j + d_{P_j^s}} \mu_{P_j^s}.
\end{equation}
The claim then follows from the stationarity assumption. 
\end{proof}

There are two main consequences of Lemma \ref{lem:limit_shape}. First, the particular shape $\mu_{P_j}$ of the 
parameter $p_j$ will be useful for the deduction of overlapping communities in the next section. More 
importantly, however, Lemma \ref{lem:limit_shape} provides us with a far-reaching interpretation of the 
quantity $\innerp{w_x}{p_j}$. Note that according to (\ref{eq:pj_recursion}), the initial values of the 
parameters $p_j$, produced by random initialization, are erased quite quickly. Indeed, after the first 
iteration, the weight of the initial value $p^1_j$ in $p^2_j$ is $\frac{1}{d_{P_j^1}}$. Since the number of 
components $k$ is usually small compared to the total degree of the graph, and the original $p_j$s are 
disjointly supported, there will be many sets $P_j^1$ with high $d_{P_j}$. Therefore $p_j$ will typically 
look like a convex combination of a few measures of the form $\mu_C$, for some subsets $C\subset V$. With this 
in mind, for a node $x \in V$ and a subset $C \subset V$, let us interpret the quantity $\innerp{w_x}
{\mu_{C}}$. 
For any $y\in V$, note that 
\begin{equation}
\mu_C(y) = \frac{|n_y \cap C|}{d_C} = \frac{\#\BRK{\mbox{edges from y to C}}}{d_C}. 
\end{equation}
Thus, 
\begin{equation}
\innerp{w_x}{\mu_{C}} = \frac{1}{d_x} \sum_{y\in n_x} \frac{|n_y \cap C|}{d_C}, 
\end{equation}
and the above sum is the number of paths of \textit{length two} from $x$ to $C$, normalized 
by the total degree of $C$. Thus, assuming $p_j$ is close to the form $\mu_C$ for some $C$, the cost 
(\ref{alg:innerp}) prefers measures $p_j$ with a large number of second order neighbours of $x$ and a small 
total degree.  

From the above discussion it follows that CLAG works by maintaining its current estimates
of the communities and aggregating nodes towards the communities that have the most similarity with the nodes. 
Of course, this is the operating scheme of many community detection algorithms. Perhaps the most close in 
spirit to our algorithm is the label propagation algorithm for non-overlapping communities due to 
\cite{Rhagavan}. In this algorithm, one simply assigns a node $x$ to a community which contains the maximal 
number of $x$'s neighbours, among all the existing communities. The above mentioned COPRA algorithm, \cite{Greg}, is a particular extension of the label propagation algorithm to overlapping communities. Some 
even earlier examples of algorithms that use ``per node" iteration schemes and neighbourhood based decisions 
are the works \cite{Jerrum} and \cite{SBB}. The distinctive feature of CLAG algorithm is that, as 
mentioned earlier, the cost $\innerp{w_x}{\mu_{C}}$ implicitly counts neighbours of $x$ at distance two rather 
then direct neighbours, and thus provides a less noisy estimate of whether node $x$ should belong to community 
$C$.

\subsection{Overlapping Communities}
\label{sec:oc}

Suppose that for a graph $G=(V,E)$ we obtained a partition $C_1,\ldots,C_k$ into disjoint communities from 
CLAG algorithm.  It is natural to assume that a node should be a member of a given community if it has 
many links to other members of this community. Specifically, for a node $x\in V$, we can define its membership 
in a given 
community using the probability to reach $x$ by a step of a random walk started at that community. In this 
way, a node can be considered a member of several communities if the probability to reach it from these 
communities is relatively high. 
We now define this formally. 

Recall that we denote by $\pi$ the stationary measure of the random walk on $G$, (\ref{eq:pi_def}).
For any partition $C_1,\ldots,C_k$ of $V$ the following decomposition holds:
\begin{equation}
	\label{eq:pi_decomposition}
	\pi = \sum_{j=1}^k \pi(C_k) \mu_{C_k}.
\end{equation}
Indeed, the right hand-side describes the distribution of a process of choosing one of the components $C_j$ at 
random (with probabilities $\pi(C_j)$), and making a step of a random walk from that component. The equality 
(\ref{eq:pi_decomposition}) then states the invariance of $\pi$ under the random walk. Conversely, suppose 
the random walk hit a node $x$. Denote by $\gamma_{x}(j)$ the probability that the component $C_j$ was chosen 
given the node $x$ was hit. Then
\begin{eqnarray}
	\label{eq:gamma_def}
	\gamma_{x}(j) =  \frac{\pi(C_j) \mu_{C_j}(x)}{  \sum_{i=1}^k \pi(C_j) \mu_{C_j}(x)}  =  \frac{\pi(C_j) \mu_{C_j}(x)}{  \pi(x)}. 	           
\end{eqnarray}
We regard the $\gamma_x(j)$'s as a probabilistic membership model. A node $x$ will be considered a member of community $C_j$ with probability $\gamma_x(j)$. 

The benchmarks for overlapping communities are usually binary, such that a node is either a member of certain community or not. We can derive such binary assignments by simple thresholding of $\gamma_x(j)$. Specifically, 
fix some threshold value $\alpha \in [0,1]$.  The value $\alpha = 0.5$ works well in most cases. For a node
$x \in V$, set $s = \argmax_{j\leq k} \gamma_x(j)$ to be the index of a community with maximal hit probability at $x$. 
Then assign $x$ to the communities $\Gamma_x$, where 
\begin{equation}
	\Gamma_x = \SETBRK{ j \leq k \setsep \gamma_x(j) \geq \alpha \gamma_x(s)}. 
\end{equation}

Up to this point we described how to obtain the overlapping communities $\Gamma_x$ from the 
measures $\mu_{C_j}$ and the weights $\pi(C_j)$. We note that there are two ways of obtaining these 
parameters from the output of CLAG. . One possibility is to directly construct the 
measures from the sets $C_j$, by iterating one time over the graph and summing the measures $w_x$ over the 
$x \in C_j$.  Alternatively, note that by Lemma \ref{lem:limit_shape}, assuming the algorithm terminated in 
a near-convergence state, the parameters $p_j$ are already in the form close to $\mu_{C_j}$ and 
the values $\hat{\pi}_j = \frac{n_j}{\sum_{i\leq k} n_i}$ approximate $\pi(C_j)$. We therefore could use 
these parameters directly in equation (\ref{eq:gamma_def}), thus avoiding an additional iteration over the graph.  Although we believe the second approach should work well, in this paper we experimented only with 
the first approach. 

For future reference, we formalize the overlapping communities algorithm as Algorithm \ref{alg:oc}, to which we refer as CLAGO.

\begin{algorithm}
   \caption{Overlapping Communities, CLAGO}
   \label{alg:oc}
\begin{algorithmic}[1]	
	\State {\bfseries Input:} Graph $G$, number of components $k$,  \\
	 \spaceo \spaceo \spaceo \spaceo threshold parameter $\alpha$. 
	\State Apply CLAG to obtain a partition $C_1,\ldots,C_k$.
	\State Compute the parameters $\pi(C_j), \mu_{C_j}$ and $\gamma_x$ using \\ 
	\spaceo \spaceo \spaceo \spaceo equations  (\ref{eq:pi_def}),(\ref{eq:mu_c_def}), and (\ref{eq:gamma_def}).
	\State Return the communities $\tilde{C_1},\ldots,\tilde{C_k}$, where \\
	 \spaceo \spaceo \spaceo \spaceo $\tilde{C_j} = \SETBRK{ x \in V \setsep j \in \Gamma_x}	$
\end{algorithmic}
\end{algorithm}

\subsection{Additional Remarks}
Most of the computation of CLAG is done in the loop of lines 5 to 9. We believe that parallelizing this loop to multiple processors should be possible, however the parallelization is not trivial. 
While performing the random walk can be done in parallel, the update to each of the $p_t$ has to be done using mini-batching and some care has to be taken in syncing the mini-batches updates. 
We leave the challenge of parallelizing the algorithm for future work.

\section{Evaluation}
In this section we present the experimental evaluation of CLAG and CLAGO algorithms. In section 
\ref{sec:std_examples} we illustrate the CLAG Algorithm on two well known small benchmark graphs. Section
\ref{sec:lfr} contains the evaluation of the CLAG algorithm on non-overlapping benchmark with parameters 
that were used in the benchmark paper \cite{LFComp}. In section \ref{sec:olfr} we provide the comparison 
of the CLAGO algorithm with the results that were given in \cite{GopBlei}.

\label{sec:eval}

\subsection{Some Standard Examples}
\label{sec:std_examples}
\begin{figure}[ht]
\vskip 0.2in
\begin{center}
\centerline{\includegraphics[width=\columnwidth]{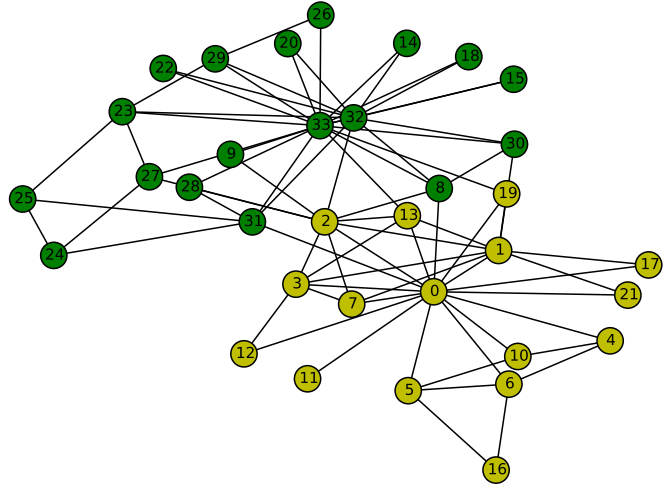}}
\caption{Karate Club Graph}
\label{fig:karate}
\end{center}
\vskip -0.2in
\end{figure}

\begin{figure}[ht]
\vskip 0.2in
\begin{center}
\centerline{\includegraphics[width=\columnwidth]{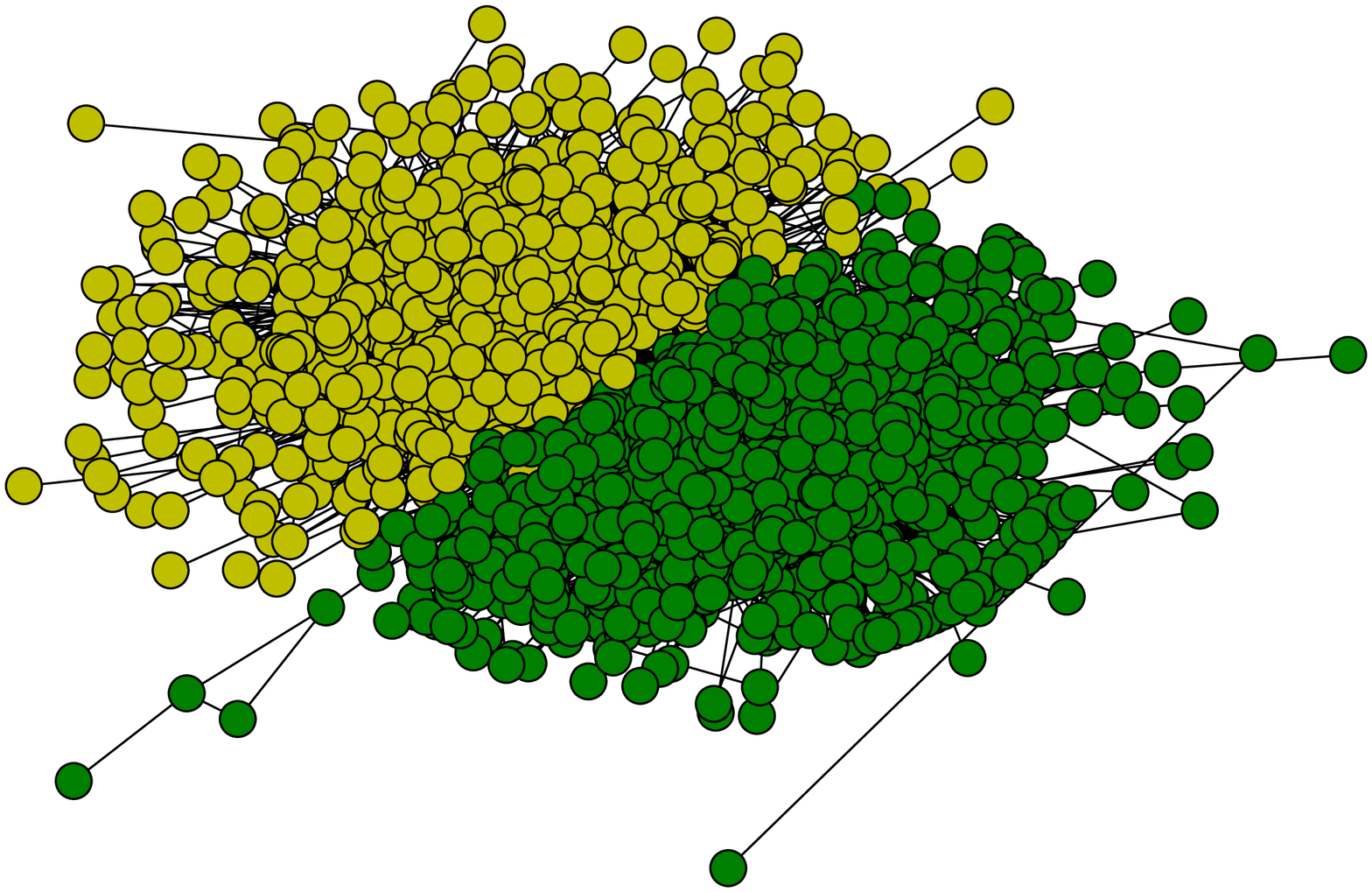}}
\caption{Political Blogs Graph}
\label{fig:political_blogs}
\end{center}
\vskip -0.2in
\end{figure} 

Figure~\ref{fig:karate} shows the classical Zachary's Karate Club graph, \cite{Zachary}. This graph has 32 nodes and a ground partition into two subsets. The partition shown in Figure \ref{fig:karate} is a partition obtained from a typical run of CLAG with $k=2$. It coincides with the ground partition except for one node, node 8, which is often miss-classified by community detection algorithms (see \cite{Fortunato201075}). We note that because the graph is small, CLAG is somewhat sensitive to the random initialization. Some invocations of the algorithm would produce the wrong partition. 
A simple common way to obtain consistent results is to restart the algorithm 3 times and to choose the partition for which some cost, such as for instance modularity,\cite{GNFoot}, is maximal.  

Figure~\ref{fig:political_blogs} depicts the political blogs graph, \cite{PolBlogs}.  The nodes are political 
blogs, and the graph has an (undirected) edge if one of the blogs had a link to the other. There are 1222 
nodes in the graph. The ground truth partition of this graph has two components - the right wing and left wing 
blogs. The labelling of the ground truth was partially automatic and partially manual, and both processes could introduce some errors. CLAG consistently reconstructs the ground truth partition with only 57 to 60 nodes misclassified. These results are similar to results obtained by other methods for this graph,~\cite{NewmanSBM}.

\subsection{LFR Benchmarks}
The LFR benchmark, \cite{LFRBench}, \cite{LFBench} is a model of a random graph with communities, such that the node degrees and community 
sizes have power law distributions, as often observed in real graphs. An important parameter of this model is the mixing 
parameter $\mu \in [0,1]$, which controls the fraction of the edges of a node that go outside the node's community (or 
outside all of node's communities, in the overlapping case). For small $\mu$, there will be a small number of edges going 
outside the communities, leading to disjoint, easily separable graphs, and the boundaries between communities will become 
less pronounced as $\mu$ grows.  The model is generated roughly as follows: First one samples the degrees from a specified 
power law and assigns them to nodes. Then one samples community sizes, from another power law, until the sizes sum up to the 
total size of the graph (or more accordingly, in the overlapping case). One can also clamp the power law, so that community 
sizes and degrees will fall in predefined boundaries. Then one connects the nodes at random, in a way that preserves degrees 
and the mixing coefficient. The paper  \cite{LFRBench} introduces the benchmarks for non-overlapping communities and in \cite{LFBench} the overlapping communities case is treated. 

The quality of communities found by an algorithm will be measured by a version of the normalized mutual information with respect to the ground truth communities. 
Given two partitions, $P,Q$ on a set $V$, the normalized mutual information between the partitions is 
\begin{equation}
\label{eq:nmi_def}
	NMI(P,Q) = 2\frac{I(P,Q)}{H(P) + H(Q)},
\end{equation}
where $H$ is the Shannon entropy of the partition and $I$ is the mutual information between the partitions (see \cite{Cover2006}, \cite{Fortunato201075}). An important property of the NMI is that it is equal to $1$ if and only if the partitions $P$ and $Q$ coincide, and it takes values between $0$ and $1$ otherwise. 

In NMI, the sets inside $P,Q$ can not overlap. An extension of NMI to the overlapping case was proposed in \cite{LFK_nmi}, and also has the property of being equal to $1$ if and only if the communities coincide. 
This extension is used to evaluate the results in \cite{GopBlei} and is also used in \cite{LFComp}, where a number of non-overlapping community detection algorithms are compared. Note that in both these papers the extended NMI is denoted by NMI. Here we will refer to the extended NMI by ENMI. The values of ENMI are usually lower then the values of the NMI. 

\subsubsection{Non-overlapping case}
\label{sec:lfr}

Figure \ref{fig:lfr_p0_enmi} shows the results of running CLAG on graphs generated from the non-overlapping LFR model, with $N=1000$ and $N=5000$ nodes, where the community sizes where allowed to range between $10$ to $50$ (denoted by $S$ in the graph), and between $20$ to $100$ (denoted by $B$), and where $\mu$ ranges between $0$ and $0.8$ in steps of $.1$. For all the graphs, the average degree is 20, the maximum degree 50, the exponent of the degree 
distribution is $-2$ and that of the community size distribution is $-1$. These parameters correspond to experiments in 
\cite{LFComp}.

The $x$ axis of Figure \ref{fig:lfr_p0_enmi} is the mixing parameter $\mu$ and the $y$ axis is the ENMI. Each point on the graph is an average of the ENMI on 20 instances of the random graphs with a given parameter set. We have not used restarts for these experiments, and we have set the number of communities $k$ to be the true number of communities for each instance. 

The results in figure \ref{fig:lfr_p0_enmi} indicate that CLAG performs better than most of the algorithms tested in \cite{LFComp}. These results are also similar to those of the Poisson model algorithm on this benchmark in \cite{NewmanPoi}. 

\begin{figure}[ht]
\vskip 0.2in
\begin{center}
\centerline{\includegraphics[width=\columnwidth]{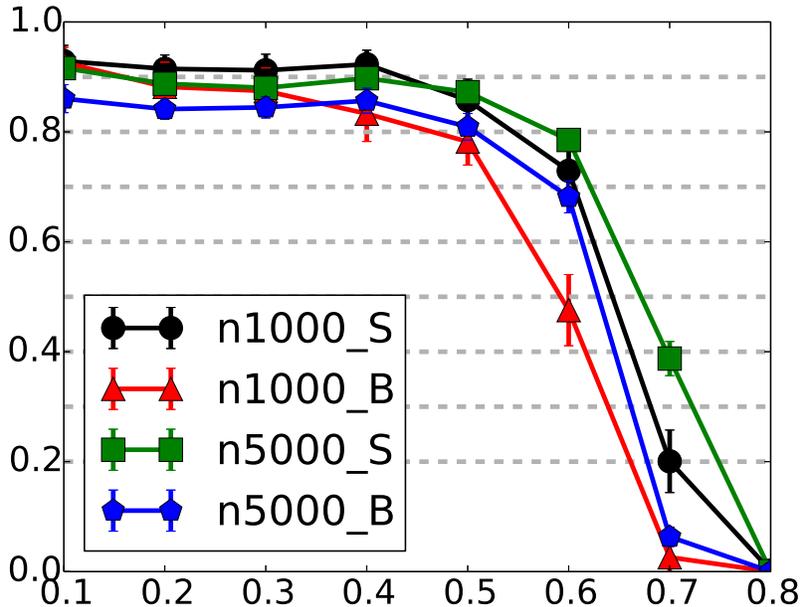}}
\caption{LFR benchmarks,ENMI }
\label{fig:lfr_p0_enmi}
\end{center}
\vskip -0.2in
\end{figure}

\subsubsection{Overlapping case}     
\label{sec:olfr}
For this section we use the overlapping LFR model, as defined in \cite{LFRBench}, with the same parameters as 
were used for evaluation in ~\cite{GopBlei}. In addition to the parameters that are present on the non-overlapping model, in the overlapping 
LFR model one specifies the parameters $\tilde{n}$ and $\tilde{m}$. $\tilde{n}$ is the number of nodes
that participate in multiple communities, and each such node will be a member of $\tilde{m}$ communities. 
The rest of the nodes will participate in exactly one community each. For all the experiments in this section, 
$\tilde{n} = N \big/ 2$, where $N$ is the total number of nodes, and $\tilde{m} = 4$. The average degree 
is $60$ for all the graphs. The maximum degree, $maxk$ , minimal and maximal community sizes, $minc$ and 
$maxc$ are functions of $N$. Their values are specified in Table \ref{tab:eval_settings}. The exponents 
of the degree and the community size distributions are the defaults, $-2$ and $-1$. For all the experiments in 
this section we have run the CLAG algorithm with $15$ iterations of the main loop, except for the
$N=1000000$ case, where we have used 5 iterations. 

There are two experiments that are performed in this section. In the first experiment we 
create graphs of size $N = 10000$, with the parameters as described above and the mixing parameter 
$\mu$ varies between $0$ and $.7$ is steps of $.1$. For each value of the parameter $\mu$ we create 
10 instances of the model and apply CLAGO. 

Note that in a non-overlapping graphs, the case $\mu = 0$ would be trivial since the communities would 
correspond simply to the connected components of the graph. However, when communities can share nodes this is 
no longer the case.  

An interesting property of CLAG that was revealed by the experiments is 
that if one starts with a number of components significantly higher then the true number, the algorithm will 
retain only the necessary number. Namely, if starting $k$ is high, then the algorithm will return with 
many empty sets $C_j$, and the number of the non-empty ones will be close to the true $k$. Therefore, we only 
need to choose high enough $k$ to start with. The number of communities of $N\leq 100000$ graphs in this 
section is strongly concentrated around 75 (it is between 72 and 78 for all instances), and for $N=1000000$ 
this number is $750$. Consequently for all graphs with $N \leq 100000$ we set $k=150$ and for $N=1000000$ set 
$k=1500$.  We note however, that this convergence of the number of components seems to depend on the 
topological complexity of the graph. It does not happen on the benchmarks of the previous section for higher 
values of $\mu$.

Another, possibly related, feature of the algorithm that was observed on the benchmarks is that some of the 
sets $C_j$ that are returned have unreasonably small sizes. For instance, for the $N=100000$ benchmarks, where 
minimal community size is $200$, some of the retuned sets where of sizes less then 10. As a general rule, 
unless one expects to have small-size communities, one can prune these sets from the final results.  

We now return to the description of the experiments.  Figure \ref{fig:olfr_mix} contains the results
of the first experiment, as described above, and shows the value of ENMI against the mixing coefficient $\mu$. Each point is an average of the evaluations over $10$ random 
instances. The standard deviation of the results at each $\mu$ was nearly zero. We show the results with and 
without pruning. With pruning, communities of sizes less then $20$ were removed from the final results. The 
performance without pruning is close but higher then all the algorithms that were considered in 
\cite{GopBlei}, and the performance with pruning improves further. The running times will be discussed 
separately in the end of this section. 

In the second experiment in this section we evaluate the performance of CLAGO on graphs 
with sizes $N=1000$,
$10000$,$10000$ and $N = 1000000$ with mixing parameter $\mu = 0$. The results are given in Table 
\ref{tab:mu0_no_prune}. Each row represents an average of $10$ instances. First two columns are the average
ENMI and the standard deviation of that average, the third row is the average number of ground truth 
communities in the graph and the last row is the average number of communities returned by the algorithm. The 
standard deviations of these averages where less then 5, and 10 for the $N=1000000$ case. 
For $N = 100000$ and $N = 1000000$ we also perform pruning at community size of 200, and Table 
\ref{tab:mu0_prune} shows the results of the same runs, after pruning the pruning. 

As mentioned earlier, for $N = 10000$ our results are close but slightly better than all the results in 
\cite{GopBlei}. For $N = 1000000$ our results are practically the same as those of the SVI and Poisson model 
algorithms, and pruning can further improve the results. For $N=100000$ our results are worse, due to 
redundancy of many sets in the returned partition, and pruning increases the ENMI significantly. An alternative to 
pruning in this case could be using lower $k$ from the start. In real world this would mean obtaining 
 a better estimate on the number of communities before running the algorithm. One could, for instance, use the 
 same procedure that is used in \cite{GopBlei} for the SVI and Poisson algorithms.

Finally, the experiments were performed on a standard PC with $i7-4770$ CPU at $3.40$GHz under Ubuntu. 
The running time for $N=1000$ cases was less then a second. The running time on the $N=10000$ instances was $5.5$ seconds per instance, for 15 iterations of the algorithm. The running time for $N=100000$ was $3$ minutes for 15 iterations. The running time for $N=1000000$ was $155$ minutes for 5 iterations of the algorithm, about 30 minutes per iteration. The time that was allocated for the $N=1000000$ case for SVI and Poisson algorithms in \cite{GopBlei} was 24 hours. We conclude that CLAGO algorithm achieves a similar or even slightly better 
performance in a significantly shorter time. 

Note that in the $N=1000000$ case, both the size of the graph and the number of communities grew by a factor of 10 compared to $N=100000$ case. The running time of the algorithm depends on the product of these quantities, which explains the jump between these two cases.

\begin{table}[ht]
\centering
\caption{Parameter Settings For LFR Graphs}
\label{tab:eval_settings}
\begin{tabular}{|l|l|l|l|} \hline
N & maxk & minc & maxc \\ \hline
1000 & 100 & 20 & 50\\ \hline
10000 & 100 & 200 & 500\\ \hline
100000 & 316 & 2000 & 5000\\ \hline
1000000 & 1000 & 2000 & 5000\\ \hline
\end{tabular}
\end{table}

\begin{figure}[ht]
\vskip 0.2in
\begin{center}
\centerline{\includegraphics[width=\columnwidth]{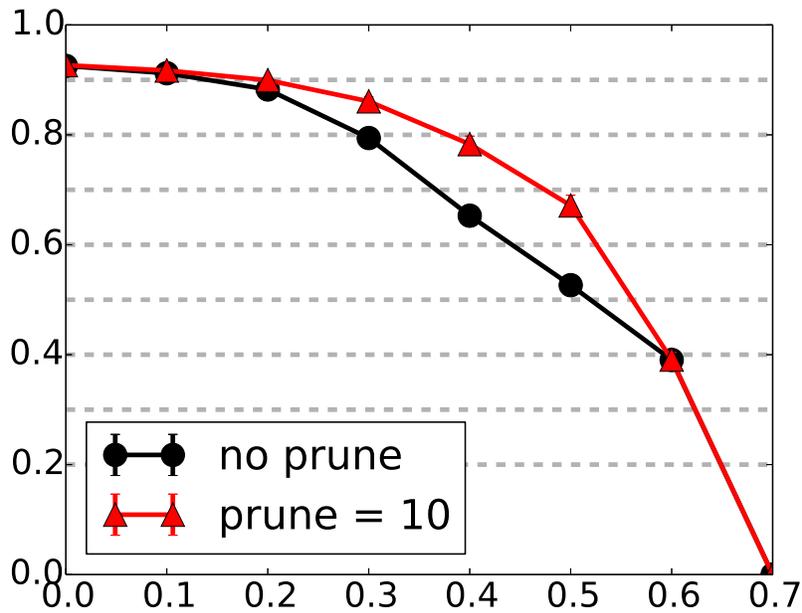}}
\caption{Overlapping LFR, N=$10,000$,ENMI}
\label{fig:olfr_mix}
\end{center}
\vskip -0.2in
\end{figure}

\begin{table}[ht]
\centering
\caption{Results for mix = 0 case}
\label{tab:mu0_no_prune}
\begin{tabular}{|l|l|l|l|l|} \hline
N & ENMI & std & orig. comm. & found comm. \\ \hline
1000 & 0.87 & 0.02 & 46.4 & 48.6\\ \hline
10000 & 0.93 & <0.01 & 75.6 & 75.2 \\ \hline
100000 & 0.63 & 0.01 & 74.3 & 149.7\\ \hline
1000000 & 0.77 & <0.01 & 759 & 811\\ \hline
\end{tabular}
\end{table}

\begin{table}[ht]
\centering
\caption{Results for mix = 0 case, with prune = 200 }
\label{tab:mu0_prune}
\begin{tabular}{|l|l|l|l|l|} \hline
N & ENMI & std & orig. comm. & found comm. \\ \hline
100000 & 0.86 & <0.01 & 74.3 & 74\\ \hline
1000000 & 0.79 & <0.01 & 759 & 773\\ \hline
\end{tabular}
\end{table}

\bibliographystyle{abbrv}
\bibliography{online_comm}  % sigproc.bib is the name of the Bibliography in this case

\end{document}